%% file: main.tex
\documentclass[runningheads]{llncs}

\usepackage[utf8]{inputenc}
\usepackage[english]{babel}
\usepackage[normalem]{ulem}

\usepackage{amsmath}
\usepackage{amssymb}
\usepackage{amsfonts}
\usepackage{xspace}
\usepackage{array}
\usepackage{multirow}
\usepackage{booktabs}
\usepackage{physics}
\usepackage{pgfplots}
\usepgfplotslibrary{groupplots}
\pgfplotsset{compat=newest}
\usepackage{dsfont}
\usepackage[numbers]{natbib}

\usepackage{amsthm}

\DeclareMathOperator*{\argmax}{argmax}
\DeclareMathOperator*{\argmin}{argmin}
\DeclareMathOperator*{\expect}{ \mathbb{E}}

\newcolumntype{C}[1]{>{\centering\arraybackslash}p{#1}}
\newcolumntype{L}[1]{>{\arraybackslash}p{#1}}

\newcommand{\Omit}[1]{}

\newcommand{\ltwo}{\ensuremath{\ell_2}\xspace}
\newcommand{\linf}{\ensuremath{\ell_\infty}\xspace}
\newcommand{\E}{\mathbb{E}}

\DeclareMathOperator*{\Vol}{Vol}
\newcommand{\prob}{\mathbb{P}}

\begin{document}

\title{Advocating for Multiple Defense Strategies against Adversarial Examples}
\titlerunning{Advocating for Multiple Defense Strategies}

\author{Alexandre Araujo$^{1,2}$\and
Laurent Meunier$^{1,3}$\and
Rafael Pinot$^{1,4}$\and \\
Benjamin Negrevergne$^{1}$}

\authorrunning{A. Araujo et al.}
\institute{PSL, Université Paris-Dauphine, Miles Team \and Wavestone \and Facebook AI Research \and CEA, Université Paris-Saclay \\
\email{\{firstname.lastname\}@dauphine.psl.eu}}

\maketitle 

\begin{abstract}
  It has been empirically observed that defense mechanisms designed to protect neural networks against $\ell_\infty$ adversarial examples offer poor performance against $\ell_2$ adversarial examples and vice versa. In this paper we conduct a geometrical analysis that validates this observation. Then, we provide a number of empirical insights to illustrate the effect of this phenomenon in practice. 
  Then, we review some of the existing defense mechanism that attempts to defend against multiple attacks by mixing defense strategies. Thanks to our numerical experiments, we discuss the relevance of this method and state open questions for the adversarial examples community.
\end{abstract}

\section{Introduction}
\label{intro}
Deep neural networks achieve state-of-the-art performances in a variety of domains such as natural language processing~\cite{radford2018Language}, image recognition~\cite{He_2016_CVPR} and speech
recognition~\cite{hinton2012deep}. However, it has been shown that such neural networks are vulnerable to {\em adversarial examples}, \emph{i.e.}, imperceptible variations of the natural examples, crafted to deliberately mislead the models~\cite{globerson2006nightmare,biggio2013evasion,Szegedy2013IntriguingPO}. Since their discovery, a variety of algorithms have been developed to generate adversarial examples (a.k.a. attacks), for example FGSM \citep{goodfellow2014explaining}, PGD \citep{madry2018towards} and C\&W \citep{carlini2017towards}, to mention the most popular ones.

Because it is difficult to characterize the space of visually imperceptible variations of a natural image, existing adversarial attacks use surrogates that can differ from one attack to another. For example, \cite{goodfellow2014explaining} use the \linf norm to measure the distance between the original image and the adversarial image whereas \cite{carlini2017towards} use the \ltwo norm.  When the input dimension is low, the choice of the norm is of little importance because the \linf and \ltwo balls overlap by a large margin, and the adversarial examples lie in the same space. An important insight in this paper is to observe that the overlap between the two balls  diminishes exponentially quickly as the dimensionality of the input space increases. For typical image datasets with large dimensionality, the two balls are mostly disjoint. As a consequence, the \linf and the \ltwo adversarial examples lie in different areas of the space, and it explains why \linf defense mechanisms perform poorly against \ltwo attacks and vice versa. 

Building on this insight, we advocate for designing models that incorporate defense mechanisms against both \linf and \ltwo attacks and review several ways of mixing existing defense mechanisms. In particular, we evaluate the performance of  {\em Mixed Adversarial Training} (MAT)~\cite{goodfellow2014explaining} which consists of  augmenting training batches using \emph{both} \linf and \ltwo adversarial examples, and {\em Randomized Adversarial Training} (RAT)~\cite{salman2019provably}, a solution to benefit from the advantages of both \linf adversarial training, and \ltwo randomized defense.

\paragraph{Outline of the paper.} The rest of this paper is organized as follows. In Section~\ref{sec:preliminaries}, we recall the principle of existing attacks and defense mechanisms. In Section~\ref{sec:no_free_lunch}, we conduct a theoretical analysis to show why  the \linf defense mechanisms cannot be robust against \ltwo attacks and vice versa. We then corroborate this analysis with empirical results using real adversarial attacks and defense mechanisms. In Section~\ref{sec:building_defense_mechanisms}, we discuss various strategies to mix defense mechanisms, conduct comparative experiments, and discuss the performance of each strategy.

\section{Preliminaries on Adversarial Attacks and Defenses}
\label{sec:preliminaries}

Let us first consider a standard classification task with an input space $\mathcal{X}=[0,1]^d$ of dimension $d$,  an output space $\mathcal{Y}=[K]$ and a data distribution $\mathcal D$ over $\mathcal X \times \mathcal Y$. We assume the model $f_\theta$ has been trained to minimize the expectation over $\mathcal{D}$ of a loss function $\mathcal{L}$ as follows:
\begin{equation}
    \min_{\theta} \E_{(x,y) \sim \mathcal{D}} \left[ \mathcal{L}(f_\theta(x), y) \right]. 
    \label{eqn:classification}
\end{equation}

\subsection{Adversarial attacks}
\label{subsec:adversarial_attacks}
 
Given an input-output pair $(x,y) \sim \mathcal{D}$, an {\em adversarial attack} is a procedure that produces a small perturbation $\tau \in  \mathcal X$  such that $f_\theta(x + \tau) \neq y$. To find the best perturbation $\tau$, existing attacks can adopt one of the two following strategies:  (i)  maximizing the loss $\mathcal L(f_\theta(x + \tau), y)$ under some constraint on $\norm{\tau}_p$\footnote{with $p \in \{0, \cdots, \infty\}$.} (a.k.a. loss maximization); or (ii)  minimizing $\norm{\tau}_p$ under some constraint on the loss $\mathcal L(f_\theta(x + \tau), y)$ (a.k.a. perturbation minimization). 

\paragraph{(i) Loss maximization.} In this scenario, the procedure maximizes the loss objective function, under the constraint that the $\ell_p$ norm of the perturbation remains bounded by some value $\epsilon$, as follows:  

\begin{equation}
  \argmax_{\norm{\tau}_p \leq \epsilon} \mathcal{L}(f_\theta(x+\tau),y).
  \label{eqn::lossmax}
\end{equation}

The typical value of $\epsilon$ depends on the norm $\norm{\cdot}_p$ considered in the problem setting. In order to compare \linf and \ltwo attacks of similar strength, we choose values of $\epsilon_\infty$ and $\epsilon_2$ (for \linf and \ltwo norms respectively) which result in \linf and \ltwo balls of equivalent volumes. For the particular case of CIFAR-10, this would lead us to choose $\epsilon_\infty = 0.03$ and $\epsilon_2 = 0.8$ which correspond to the maximum values chosen empirically to avoid the generation of visually detectable perturbations. 
The current state-of-the-art method to solve Problem~(\ref{eqn::lossmax}) is based on a projected gradient descent (PGD)~\cite{madry2018towards} of radius~$\epsilon$. Given a budget $\epsilon$, it recursively computes
\begin{equation}
    x^{t+1}=\prod_{B_p(x,\epsilon)}\left(x^t
+\alpha \argmax_{\delta\text{ s.t. }||\delta||_p\leq1} \left(\Delta^t|\delta \right)\right)
    \label{eqn::projectionPGD}
\end{equation}
where $B_p(x,\epsilon) = \{ x+\tau \text{~s.t.~} \norm{\tau}_p \leq \epsilon\}$, $\Delta^t=\nabla_x\mathcal{L}\left(f_\theta\left(x^t\right),y\right)$, $\alpha$ is a gradient step size, and $\prod_S$ is the projection operator on $S$. Both PGD attacks with $p=2$, and $p=\infty$ are currently used in the literature as state-of-the-art attacks for the loss maximization problem.

\paragraph{(ii) Perturbation minimization.}  This type of procedure search for the perturbation that has the minimal $\ell_p$ norm, under the constraint that $\mathcal{L}(f_\theta(x+\tau),y)$ is bigger than a given bound $c$:
    \begin{equation}
      \argmin_{\mathcal{L}(f_\theta(x+\tau),y) \geq c} 
      \norm{\tau}_p.
      \label{eqn::normmin}
  \end{equation}
  The value of $c$ is typically chosen depending on the loss function $\mathcal{L}$\footnote{For example, if $\mathcal{L}$ is the $0/1$ loss, any $c>0$ is acceptable.}.
  Problem~(\ref{eqn::normmin}) has been tackled in~\cite{carlini2017towards}, leading to the following method, denoted C\&W attack in the rest of the paper. It aims at solving the following Lagrangian relaxation of Problem~(\ref{eqn::normmin}):
  \begin{equation}
    \argmin_{\tau} \norm{\tau}_p+ \lambda \times g(x+\tau)
    \label{eqn::CWproblem}
\end{equation}
where $g(x+\tau)<0$ if and only if $\mathcal{L}(f_\theta(x+\tau),y) \geq c$. 
The authors use a change of variable $\tau=\tanh(w)-x$ to ensure that $-1 \leq x+\tau \leq 1$, a binary search to optimize the constant $c$, and Adam or SGD to compute an approximated solution. The C\&W attack is well defined both for $p=2$, and $p=\infty$, but there is a clear empirical gap of efficiency in favor of the \ltwo attack.

In this paper, we focus on the {\em Loss Maximization} setting using the PGD attack. However we  conduct some of our experiments using {\em Perturbation Minimization} algorithms such as C\&W to capture more detailed information about the location of adversarial examples in the vector space\footnote{As it has a more flexible geometry than the {\em Loss Maximization} attacks.}. 

\subsection{Defense mechanisms}
\label{subsec:defense_mechanisms}

\paragraph{Adversarial Training (AT).}
\label{paragraph:adversarial_training}

Adversarial Training was introduced in \cite{goodfellow2014explaining} and later improved in \cite{madry2018towards} as a first defense mechanism to train robust neural networks. It consists in augmenting training batches with adversarial examples generated during the training procedure. The standard training procedure from Equation~(\ref{eqn:classification}) is thus replaced by the following  $\min$ $\max$ problem, where the classifier tries to minimize the expected loss under maximum perturbation of its input:
\begin{equation}
    \min_{\theta}\expect_{(x, y)\sim \mathcal{D}} \left[ \max_{\norm{\tau}_p \leq \epsilon} \mathcal{L} \left( f_{\theta}(x+\tau), y \right) \right].
\end{equation}
\noindent
In the case where $p=\infty$, this technique offers good robustness  against $\ell_\infty$ attacks \cite{athalye2018obfuscated}. AT can also be used with \ltwo attacks but as we will discuss in Section~\ref{sec:no_free_lunch}, AT with one norm offers poor protection against the other.
The main weakness of Adversarial Training is its lack of formal guarantees. Despite some recent work providing great insights \cite{sinha2017certifying,zhang2019theoretically}, there is no worst case lower bound yet on the accuracy under attack of this method.

\paragraph{Noise injection mechanisms (NI).}\label{subsec:randomized_training}

Another important technique to defend against adversarial examples is to use Noise Injection. 
In contrast with Adversarial Training, Noise Injection mechanisms are usually deployed after training. In a nutshell, it works as follows. At inference time, given a unlabeled sample $x$, the network outputs
\begin{equation}
    \tilde{f}_\theta (x):= f_\theta(x + \eta) \ \ \ (\text{instead of  }f_\theta(x)) 
\end{equation}
where $\eta$ is a random variable on $\mathbb{R}^d$.
Even though, Noise Injection is often less efficient than Adversarial Training in practice (see \emph{e.g.}, Table~\ref{tab:results}), it benefits from strong theoretical background. In particular, recent works \cite{lecuyer2018certified,NIPS2019_9143}, followed by \cite{KolterRandomizedSmoothing,pinot2019theoretical} demonstrated that noise injection from a Gaussian distribution can give provable defense against \ltwo adversarial attacks. In this work, besides the classical Gaussian noises already investigated in previous works, we evaluate the efficiency of Uniform distributions to defend against \ltwo adversarial examples.

\section{No Free Lunch for Adversarial Defenses}
\label{sec:no_free_lunch}

In this Section, we show both theoretically and empirically that defenses mechanisms intending to defend against $\ell_\infty$ attacks cannot provide suitable defense against $\ell_2$ attacks. Our reasoning is perfectly general; hence we can similarly demonstrate the reciprocal statement, but we focus on this side for simplicity. 

\begin{figure*}[ht]
  \centering
  \begin{minipage}{.32\linewidth}
    \centering
    \includegraphics[scale=0.15]{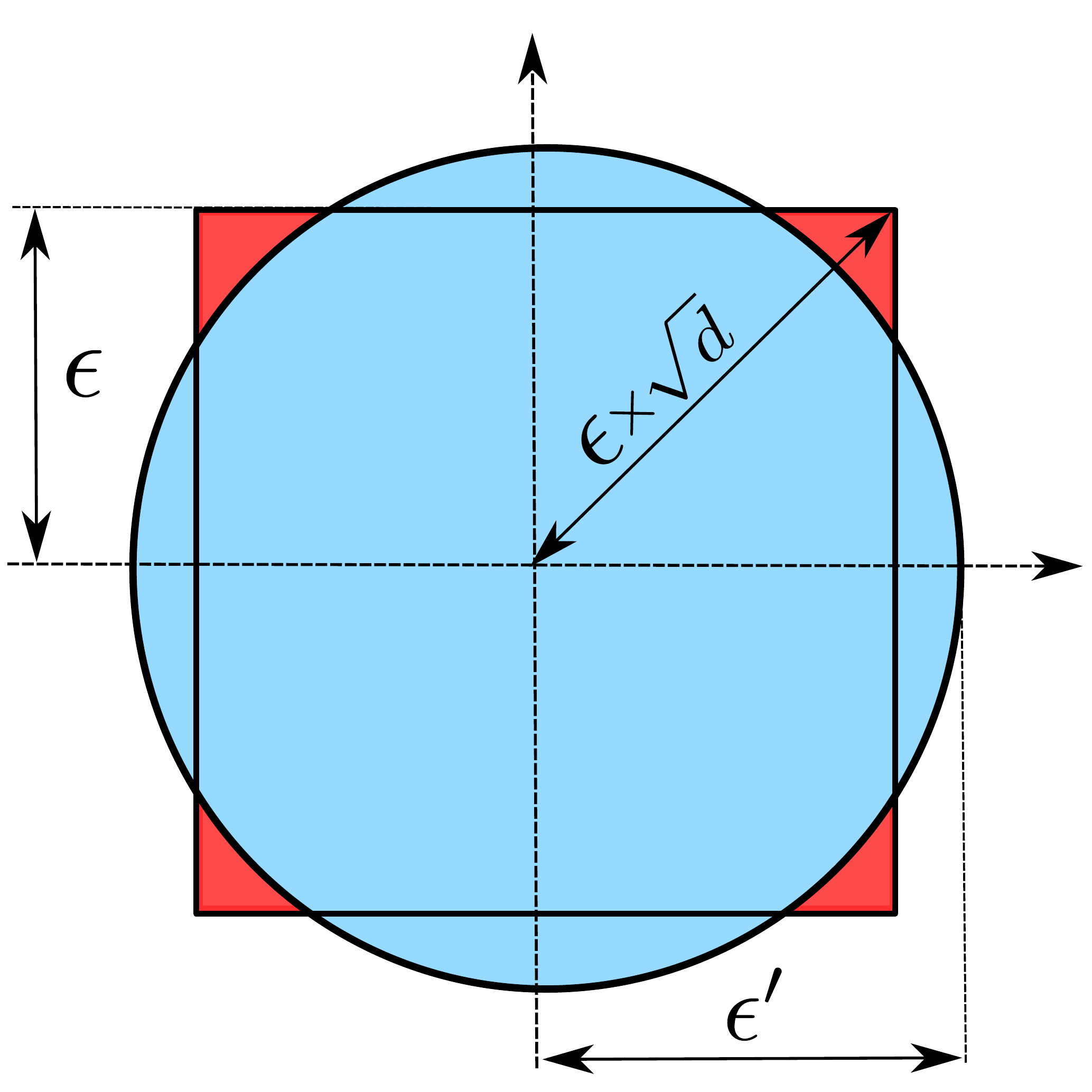}\\(a)
  \end{minipage}
  \begin{minipage}{.32\linewidth}
    \centering
    \includegraphics[scale=0.15]{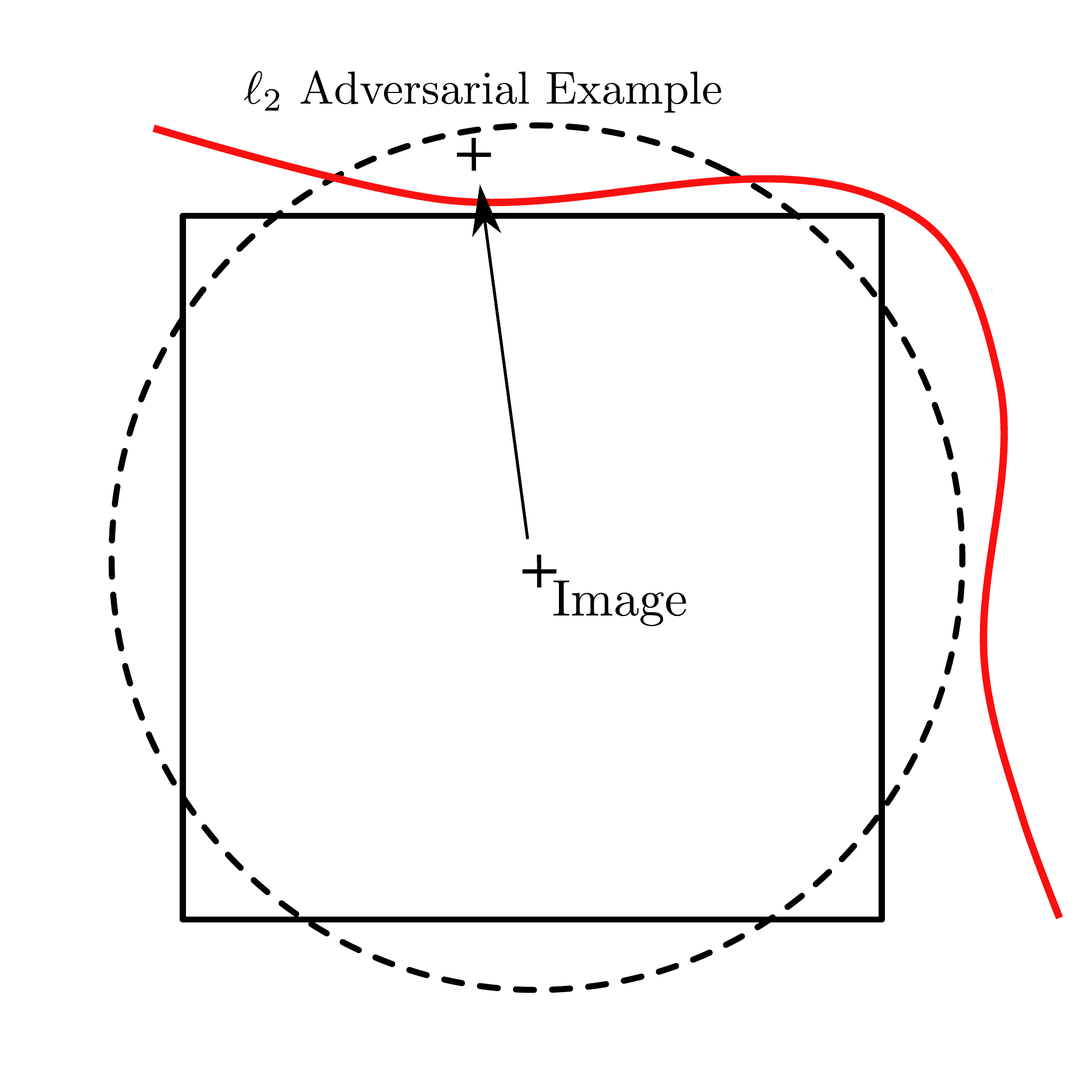}\\(b)
  \end{minipage}
  \begin{minipage}{.32\linewidth}
      \centering
      \includegraphics[scale=0.15]{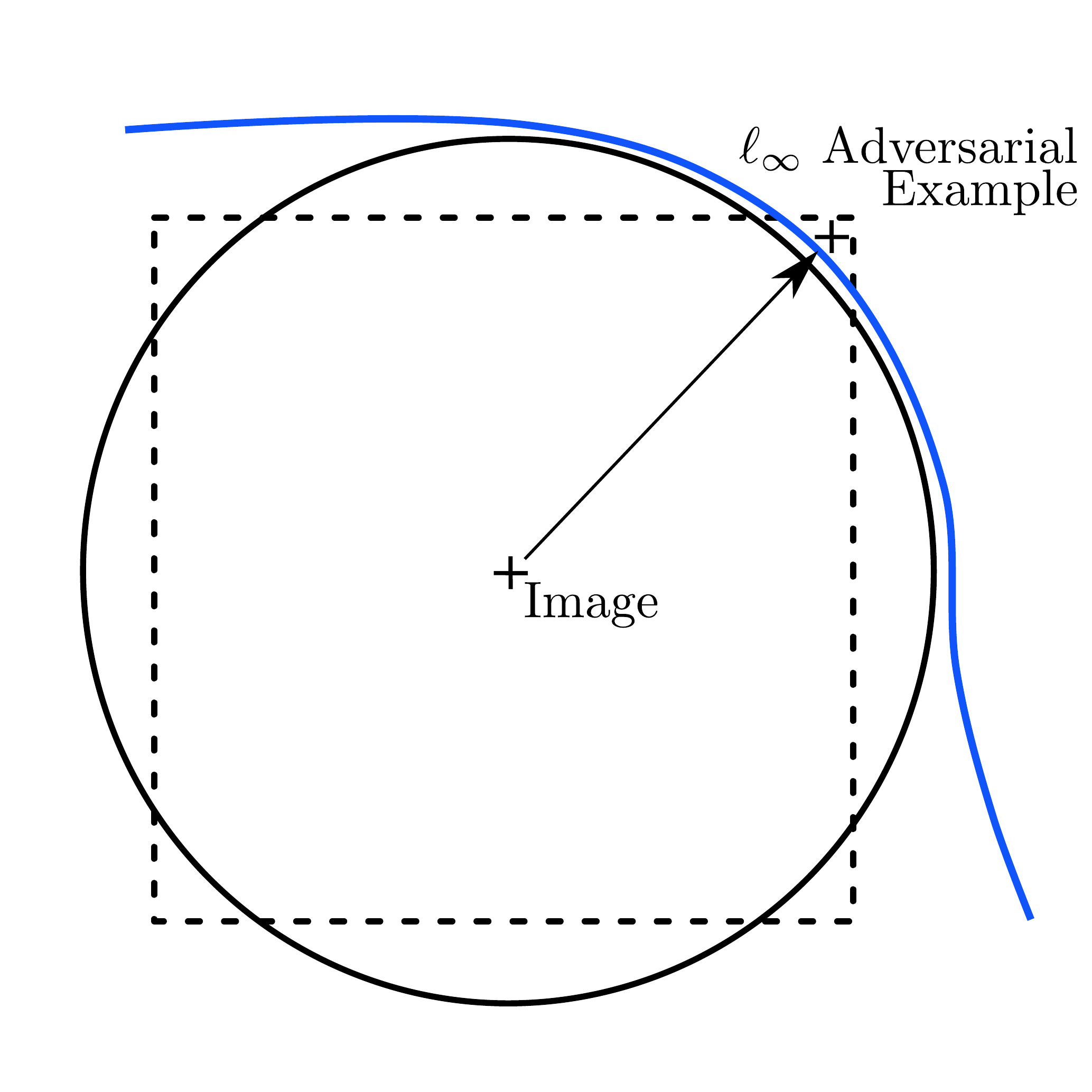}\\(c)
  \end{minipage}
    \caption{ Left: 2D representation of the \linf and \ltwo balls of respective radius $\epsilon$ and $\epsilon'$. 
    Middle: a classifier trained with \linf adversarial perturbations  (materialized by the red line) remains vulnerable to \ltwo attacks. 
    Right: a classifier trained with \ltwo adversarial perturbations (materialized by the blue line) remains vulnerable to \linf attacks.}
  \label{figure:balls}
\end{figure*}%

\subsection{Theoretical analysis}

Let us consider a classifier $f_{\infty}$ that is provably robust against adversarial examples with maximum $\ell_\infty$ norm of value $\epsilon_\infty$. It guarantees that for any input-output pair $(x,y) \sim \mathcal D$ and for any perturbation $\tau$ such that $\norm{\tau}_\infty \leq \epsilon_\infty$, $f_{\infty}$ is not misled by the perturbation, \emph{i.e.}, $f_{\infty}(x + \tau) = f_{\infty}(x)$.
We now focus our study on the performance of this classifier against adversarial examples bounded with a \ltwo norm of value $\epsilon_2$. Using Figure~\ref{figure:balls}(a), we observe that any \ltwo adversarial example that is also in the \linf ball, will not fool $f_{\infty}$. Conversely, if it is outside the ball, we have no guarantee.

To characterize the probability that such an  $\ell_2$ perturbation fools an $\ell_\infty$ defense mechanism in the general case (\emph{i.e.}, any dimension $d$), we measure the ratio between the volume of the intersection of the $\ell_\infty$ ball of radius $\epsilon_\infty$ and the $\ell_2$ ball of radius $\epsilon_2$. As Theorem~\ref{theorem:nullvolume} shows, this ratio depends on the dimensionality $d$ of the input vector $x$, and  rapidly converges to zero when $d$ increases. 
Therefore a defense mechanism that protects against all \linf bounded adversarial examples is unlikely to be efficient against \ltwo attacks.

\begin{theorem}[Probability of the intersection goes to $0$] \\
\label{theorem:nullvolume}
\noindent Let $B_{2,d}(\epsilon) :=\left \{\tau \in \mathbb{R}^d \text{ s.t } \norm{\tau}_2 \leq  \epsilon \right \}$ and $B_{\infty,d}(\epsilon') :=\left \{\tau \in \mathbb{R}^d \text{ s.t } \norm{\tau}_\infty \leq  \epsilon' \right\}$. If for all $d$, we select $\epsilon$ and $\epsilon$' such that $\Vol\left(B_{2,d}(\epsilon)\right) =\Vol\left(B_{\infty,d}(\epsilon')\right)$, then $$\frac{\Vol\left(B_{2,d}(\epsilon)\bigcap B_{\infty,d}(\epsilon')\right)}{\Vol\left(B_{\infty,d}(\epsilon')\right)} \rightarrow 0 \text{ when } d\rightarrow \infty. $$
\end{theorem} 
\begin{proof} 
Without loss of generality, let us fix $\epsilon=1$. One can show that for all $d$, 
\begin{equation}
    \Vol\left( B_{2,d}\left(\frac{2}{\sqrt{\pi}}\Gamma\left(\frac{d}{2}+1\right)^{1/d}\right)\right) = \Vol\left(B_{\infty,d}\left(1\right)\right)
\end{equation}
where $\Gamma$ is the gamma function. Let us denote 
\begin{equation}
    r_2(d)=\frac{2}{\sqrt{\pi}}\Gamma\left(\frac{d}{2}+1\right)^{1/d}.
\end{equation}
Then, thanks to Stirling's formula
\begin{equation}
    r_2(d)\sim \sqrt{\frac{2}{\pi e}} d^{1/2}.
\end{equation}
Finally, if we denote $\mathcal{U}_S$, the uniform distribution on set $S$, by using  Hoeffding inequality between Equation~\ref{eq:Hoeffding1} and \ref{eq:Hoeffding2}, we get:
\begin{align}
&\frac{\Vol(B_{2,d}(r_2(d))\bigcap B_{\infty,d}(1))}{\Vol(B_{\infty,d}(1))} \\
=&\prob_{x\sim \mathcal{U}_{B_{\infty,d}(1)}}\left[x\in B_{2,d}(r_2(d))\right] \\
=&\prob_{x\sim \mathcal{U}_{B_{\infty,d}(1)}}\left[\textstyle \sum_{i=1}^d |x_i|^2\leq r_2^2(d)\right] \\
\leq &\exp{- d^{-1} \left( r_2^2(d)-d\mathbb{E}|x_1|^2\right)^2} \label{eq:Hoeffding1} \\
\leq &\exp{-\left( \frac{2}{\pi e}-\frac13\right)^2d+ o(d)} \label{eq:Hoeffding2}.
\end{align}
\noindent
Then the ratio between the volume of the intersection of the ball and the volume of the ball converges towards $0$ when $d$ goes to $\infty$.
\end{proof}

Theorem~\ref{theorem:nullvolume} states that, when $d$ is large enough, \ltwo bounded perturbations have a null probability of being also in the \linf ball of the same volume. As a consequence, for any value of $d$ that is large enough, a defense mechanism that offers full protection against $\linf$ adversarial examples is not guaranteed to offer any protection against $\ltwo$ attacks\footnote{Th. \ref{theorem:nullvolume} can easily be extended to any two balls with different norms. For clarity, we restrict to the case of \linf and \ltwo norms.}.

\begin{table}[ht]
\centering
\caption{ Bounds of Theorem~\ref{theorem:nullvolume} on the volume of the intersection of  $\ell_2$ and $\ell_\infty$ balls at equal volume for typical image classification datasets. When $d=2$, the bound is $ 10^{-0.009}\approx 0.98$.}
\begin{tabular}{c r r r l}
\toprule
\textbf{Dataset\ } & \phantom{....} & \textbf{Dim.} $\mathbf{(d)}$ & \phantom{....} & \textbf{Vol. of the intersection }\\
\midrule
-- & & 2\ \ & & $10^{-0.009}$ \quad ($\approx$ 0.98) \\
MNIST & & 784\ \  & & $10^{-144}$\\
CIFAR & & 3072\ \ & &  $10^{-578}$\\
ImageNet & & 150528\ \ & & $10^{-28946}$\\
\bottomrule
\end{tabular}
\label{table:datadim}
\end{table}

Note that this result defeats the 2-dimensional intuition: if we consider a 2 dimensional problem setting, the \linf and the \ltwo balls have an important overlap (as illustrated in Figure~\ref{figure:balls}(a)) and the probability of sampling at the intersection of the two balls is bounded by approximately 98\%. However, as we increase the dimensionality $d$, this probability quickly becomes negligible, even for very simple image datasets such as MNIST. An instantiation of  the bound for classical image datasets is presented in Table~\ref{table:datadim}. The probability of sampling at the intersection of the \linf and \ltwo balls is close to zero for any realistic image setting. In large dimensions, the volume of the corner of the \linf ball is much bigger than it appears in Figure~\ref{figure:balls}(a).

\subsection{No Free Lunch in Practice}

Our theoretical analysis shows that if adversarial examples were uniformly distributed in a high-dimensional space, then any mechanism that perfectly defends against \linf adversarial examples has a null probability of protecting against \ltwo-bounded adversarial attacks. Although existing defense mechanisms do not necessarily assume such a distribution of adversarial examples, we demonstrate that whatever distribution they use, it offers no favorable bias with respect to the result of Theorem~\ref{theorem:nullvolume}. 
As we discussed in Section~\ref{sec:preliminaries}, there are two distinct attack settings: loss maximization (PGD) and perturbation minimization (C\&W). Our analysis is mainly focusing on loss maximization attacks. However, these attacks have a very strict geometry\footnote{Due to the projection operator, all PGD attacks saturate the constraint, which makes them all lies in a very small part of the ball.}. This is why, to present a deeper analysis of the behavior of adversarial attacks and defenses, we also present a set of experiments that use perturbation minimization attacks.

\begin{table}[htbp]
  \centering 
  \caption{Average norms of PGD-\ltwo and PGD-\linf adversarial examples with and without \linf adversarial training on CIFAR-10 ($d=3072$).}
    \begin{tabular}{lrrrrrrrr}
    \toprule
      & \phantom{...}  & \multicolumn{3}{c}{Attack PGD-\ltwo} & \phantom{...}  & \multicolumn{3}{c}{Attack PGD-\linf} \\
\cmidrule{3-5}\cmidrule{7-9}      &   & \multicolumn{1}{l}{Unprotected} &  \phantom{...} & \multicolumn{1}{l}{AT-\linf} &   & \multicolumn{1}{l}{Unprotected} & \phantom{...}  & \multicolumn{1}{l}{AT-\ltwo} \\
    \midrule
    Average \ltwo norm &   & 0.830 &   & 0.830 &   & 1.400 &   & 1.640 \\
    Average \linf norm &   & 0.075 &   & 0.200 &   & 0.031 &   & 0.031 \\
    \bottomrule \\
    \end{tabular}%
  \label{tab:mean_norm_pgd_attack_ben}%
\end{table}%

\paragraph{Adversarial training vs. loss maximization attacks}

To demonstrate that \linf adversarial training is not robust against PGD-\ltwo attacks we measure the evolution of \ltwo norm of adversarial examples generated with PGD-\linf between an unprotected model and a model trained with AT-\linf, \emph{i.e.}, AT where adversarial examples are generated with PGD-\linf \footnote{To do so, we use the same experimental setting as in Section~\ref{sec:building_defense_mechanisms} with $\epsilon_\infty$ and $\epsilon_2$ such that the volumes of the two balls are equal.}. 
Results are presented in  Table~\ref{tab:mean_norm_pgd_attack_ben}. \footnote{All experiments in this section are conducted on CIFAR-10, and the experimental setting is fully detailed in Section~\ref{sec:experimental_settings}. }

The analysis is unambiguous: the average \linf norm of a bounded \ltwo perturbation more than double between an unprotected model and a model trained with AT PGD-\linf. This phenomenon perfectly reflects the illustration of Figure~\ref{figure:balls} (c). The attack will generate an adversarial example on the corner of the \linf ball thus increasing the \linf norm while maintaining the same \ltwo norm. 
We can observe the same phenomenon with AT-\ltwo against PGD-\linf attack (see Figure~\ref{figure:balls} (b) and Table \ref{tab:mean_norm_pgd_attack_ben}). PGD-\linf attack increases the \ltwo norm while maintaining the same \linf perturbation thus generating the perturbation in the upper area. 

As a consequence, we cannot expect adversarial training \linf to offer any guaranteed protection against \ltwo adversarial examples .

\paragraph{Adversarial training vs. perturbation minimization attacks.}
To better capture the behavior of \ltwo adversarial examples, we now study the performances of an \ltwo perturbation minimization attack (C\&W) with and without AT-\linf. It allows us to understand in which area C\&W discovers adversarial examples and the impact of AT-\linf. In high dimensions, the red corners (see Figure~\ref{figure:balls} (a)) are very far away from the \ltwo ball. Therefore, we hypothesize that a large proportion of the \ltwo adversarial examples will remain unprotected. To validate this assumption, we measure the proportion of adversarial examples inside of the \ltwo ball before and after \linf adversarial training. The results are presented in Figure~\ref{fig:calotte} (left: without adversarial training, right: with adversarial training). 

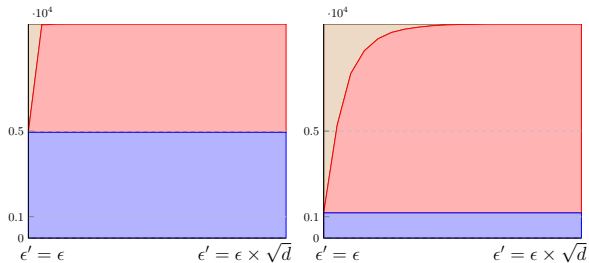
\begin{figure}[htb]
    \centering
    \input{graphs/graphs.tex}
    \caption{Comparison of the number of adversarial examples found by C\&W, inside the \linf ball (lower, blue area), outside the \linf ball but inside the \ltwo ball (middle, red area) and outside the \ltwo ball (upper gray area). $\epsilon$ is set to $0.3$ and $\epsilon'$ varies along the x-axis. Left: without adversarial training, right: with adversarial training. Most adversarial examples have shifted from the \linf ball to the cap of the \ltwo ball, but remain at the same \ltwo distance from the original example.}
    \label{fig:calotte}
\end{figure}

On both charts, the blue area represents the proportion of adversarial examples that are inside the \linf ball. The red area represents the adversarial examples that are outside the \linf ball but still inside the \ltwo ball (valid \ltwo adversarial examples). Finally, the brown-beige area represents the adversarial examples that are beyond the \ltwo bound. The radius $\epsilon'$ of the \ltwo ball varies along the x-axis from $\epsilon'$ to $\epsilon' \sqrt{d}$. On the left chart (without adversarial training) most \ltwo adversarial examples generated by C\&W are inside both balls. On the right chart most of the adversarial examples have been shifted out the \linf ball. This is the expected consequence of \linf adversarial training. However, these adversarial examples remain in the \ltwo ball, \emph{i.e.}, they are in the cap of the \ltwo ball. These examples are equally good from the \ltwo perspective. This means that even after adversarial training, it is still easy to find good \ltwo adversarial examples, making the \ltwo robustness of AT-\linf almost null.

\section{Reviewing Defenses Against Multiple Attacks}
\label{sec:building_defense_mechanisms}

\begin{table}[htbp]
  \centering
  \caption{This table shows a comprehensive list of results consisting of the accuracy of several defense mechanisms against $\ell_2$ and $\ell_\infty$ attacks. This table main objective is to compare the overall performance of ‘single‘ norm defense mechanisms (AT and NI presented in the Section~\ref{subsec:defense_mechanisms}) against mixed norms defense mechanisms (MAT \& RAT mixed defenses presented in Section~\ref{sec:building_defense_mechanisms}).}
    \begin{tabular}{lccccccccccccccccc}
    \toprule
      &   & \textbf{Baseline} & \phantom{...}  & \multicolumn{2}{c}{\textbf{AT}} & \phantom{...}  & \multicolumn{2}{c}{\textbf{MAT}} &  \phantom{...} & \multicolumn{2}{c}{\textbf{NI}} &  \phantom{...} & \multicolumn{2}{c}{\textbf{RAT}-$\ell_\infty$} &  \phantom{...} & \multicolumn{2}{c}{\textbf{RAT}-$\ell_2$} \\
\cmidrule{3-3}\cmidrule{5-6}\cmidrule{8-9}\cmidrule{11-12}\cmidrule{14-15}\cmidrule{17-18}      &   & -- &   & $\ell_\infty$ & $\ell_2$ &   & Max & Rand &   & $\mathcal{N}$ & $\mathcal{U}$ &   & $\mathcal{N}$ & $\mathcal{U}$ &   & $\mathcal{N}$ & $\mathcal{U}$ \\
    \midrule
    Natural &   & 0.94 &   & 0.85 & 0.85 &   & 0.80 & 0.80 &   & 0.79 & 0.87 &   & 0.74 & 0.80 &   & 0.79 & 0.87 \\
    PGD-$\ell_\infty$ &   & 0.00 &   & 0.43 & 0.37 &   & 0.37 & 0.40 &   & 0.23 & 0.22 &   & 0.35 & 0.40 &   & 0.23 & 0.22 \\
    PGD-$\ell_2$ &   & 0.00 &   & 0.37 & 0.52 &   & 0.50 & 0.55 &   & 0.34 & 0.36 &   & 0.43 & 0.39 &   & 0.34 & 0.37 \\
    \bottomrule
    \end{tabular}%
  \label{tab:results}
\end{table}%

Adversarial attacks have been an active topic in the machine learning community since their discovery~\cite{globerson2006nightmare, biggio2013evasion,Szegedy2013IntriguingPO}. Many attacks have been developed. Most of them solve a loss maximization problem with either $\ell_\infty$~\cite{goodfellow2014explaining,kurakin2016adversarial,madry2018towards}, $\ell_2$~\cite{carlini2017towards,kurakin2016adversarial,madry2018towards}, $\ell_1$~\cite{tramer2019adversarial} or $\ell_0$~\cite{papernot2016limitations} surrogate norms. As we showed, these norms are really different in high dimension. Hence, defending against one norm-based attack is not sufficient to protect against another one. 
In order to solve this problem, we review several strategies to build defenses against multiple adversarial attacks. These strategies are based on the idea that both types of defense must be used simultaneously in order for the classifier to be protected against multiple attacks. The detailed description of the experimental setting is described in Section~\ref{sec:experimental_settings}.

\subsection{Experimental Setting}
\label{sec:experimental_settings}

To compare the robustness provided by the different defense mechanisms, we use strong adversarial attacks and a conservative setting: the attacker has a total knowledge of the parameters of the model (white-box setting) and we only consider untargeted attacks  (a misclassification from one target to any other will be considered as adversarial). To evaluate defenses based on Noise Injection, we use {\em Expectation Over Transformation} (EOT), the rigorous experimental protocol  proposed by \cite{athalye2017synthesizing} and later used by \cite{athalye2018obfuscated,carlini2019evaluating} to identify flawed defense mechanisms. 

To attack the models, we use state-of-the-art algorithms PGD. We run PGD with 20 iterations to generate adversarial examples and with 10 iterations when it is used for adversarial training. The maximum \linf bound is fixed to $0.031$ and the maximum \ltwo bound is fixed to $0.83$. As discussed in Section~\ref{sec:preliminaries}, we chose these values so that the \linf and the \ltwo balls have similar volumes. Note that $0.83$ is slightly above the values typically used in previous publications in the area, meaning the attacks are stronger, and thus  more difficult to defend against.

All experiments are conducted on CIFAR-10 with the Wide-Resnet 28-10 architecture. We use the training procedure and the hyper-parameters described in the original paper by~\cite{zagoruyko2016wide}. Training time varies from 1 day (AT) to 2 days (MAT) on 4 GPUs-V100 servers.

\subsection{MAT -- Mixed Adversarial Training}\label{subsec:mixed_adversarial_training}
Earlier results have shown that AT-$\ell_p$ improves the robustness against corresponding $\ell_p$-bounded adversarial examples, and the experiments we present in this section corroborate this observation (See Table~\ref{tab:results}, column: AT). Building on this, it is natural to examine the efficiency of \emph{Mixed Adversarial Training} (MAT) against mixed \linf and \ltwo attacks. MAT is a variation of AT that uses both \linf-bounded adversarial examples and \ltwo-bounded adversarial examples as training examples. As discussed in~\cite{tramer2019adversarial}, there are several possible strategies to mix the adversarial training examples. The first strategy (MAT-Rand) consists in randomly selecting one adversarial example among the two most damaging \linf and \ltwo, and to use it as a training example, as described in Equation~(\ref{eq:mat-rand}): 
\paragraph{MAT-Rand}:
\begin{equation}
    \min_{\theta}\expect_{(x, y) \sim \mathcal{D}} \left[\expect_{p\sim\mathcal{U}({\{2, \infty\})}} \max_{\norm{\tau}_p \leq \epsilon} \mathcal{L} \left( f_{\theta}(x+\tau), y \right) \right].
    \label{eq:mat-rand}
\end{equation}

\noindent
An alternative strategy is to systematically train the model with the most damaging adversarial example (\linf or \ltwo). As described in Equation~(\ref{eq:mat-max}): 

\paragraph{MAT-Max}:
\begin{equation}
    \min_{\theta}\expect_{(x, y) \sim \mathcal{D}} \left[ \max_{p \in \{2, \infty\}} \max_{\norm{\tau}_p \leq \epsilon} \mathcal{L} \left( f_{\theta}(x+\tau), y \right) \right].
    \label{eq:mat-max}
\end{equation}

\noindent
The accuracy of MAT-Rand and MAT-Max are reported in Table~\ref{tab:results} (Column: MAT). As expected, we observe that MAT-Rand and MAT-Max offer better robustness both against PGD-\ltwo and PGD-\linf adversarial examples than the original AT does. More  generally, we can see that AT is a good strategy against loss maximization attacks, and thus it is not surprising that MAT is a good strategy against mixed loss maximization attacks. However efficient in practice, MAT (for the same reasons as AT) lacks theoretical arguments. In order to get the best of both worlds, \cite{salman2019provably} proposed to mix adversarial training with randomization.

\subsection{RAT -- Randomized Adversarial Training}\label{subsec:randomized_adversarial_training}

We now examine the performance of Randomized Adversarial Training (RAT) first introduced in~\cite{salman2019provably}. This technique mixes Adversarial Training with Noise Injection. The corresponding loss function is defined as follows: \begin{equation}
    \min_{\theta}\expect_{(x, y) \sim \mathcal{D}} \left[ \max_{\norm{\tau}_p \leq \epsilon} \mathcal{L} \left( \tilde{f}_{\theta}(x+\tau), y)  \right) \right].
\end{equation}
\noindent where $\tilde{f}_\theta$ is a randomized neural network with noise injection as described in Section~\ref{subsec:randomized_training}, and $\norm{\cdot}_p$ define which kind of AT is used. For each setting, we consider two noise distributions, Gaussian and Uniform as we did with NI. We also consider two different Adversarial training AT-\linf as well as AT-\ltwo. 

The results of RAT are reported in Table~\ref{tab:results}~(Columns: RAT-\linf and RAT-\ltwo).
We can observe that RAT-\linf offers the best extra robustness with both noises, which is consistent with previous experiments, since AT is generally more effective against \linf attacks whereas NI is more effective against \ltwo-attacks. Overall, RAT-\linf and a noise from uniform distribution offers the best performances but is still weaker than MAT-Rand.
These results are also consistent with the literature, since adversarial training (and its variants) is the best defense against adversarial examples so far.

\section{Conclusion \& Perspective}
In this paper, we tackled the problem of protecting neural networks against multiple attacks crafted from different norms. We demonstrated and gave a geometrical interpretation to explain why most defense mechanisms can only protect against one type of attack. Then we reviewed existing strategies that mix defense mechanisms in order to build models that are robust against multiple adversarial attacks. We conduct a rigorous and full comparison of {\em Randomized Adversarial Training} and {\em Mixed Adversarial Training} as defenses against multiple attacks. 

We could argue that both techniques offer benefits and limitations. We have observed that MAT offers the best empirical robustness against multiples adversarial attacks but this technique is computationally expensive which hinders its use in large-scale applications. Randomized techniques have the important advantage of providing theoretical guarantees of robustness and being computationally cheaper. However, the certificate provided by such defenses is still too small for strong attacks. Furthermore, certain Randomized defenses also suffer from the curse of dimensionality as recently shown by \cite{kumar2020curse}. 

Although, randomized defenses based on noise injection seem limited in terms of accuracy under attack and scalability, they could be improved either by Learning the best distribution to use or by leveraging different types of randomization such as discrete randomization first proposed in \cite{pinot2020randomization}. We believe that these certified defenses are the best solution to ensure the robustness of classifiers deployed into real-world applications.

\section{Acknowledgement}
This work was granted access to the HPC resources of IDRIS under the allocation 2020-101141 made by GENCI. We would like to thank Jamal Atif, Florian Yger and Yann Chevaleyre for their valuable insights.

\bibliographystyle{abbrv}
\bibliography{bibliography}

\end{document}

%% file: graphs/graphs.tex
\begin{tikzpicture}[scale=0.5]
    \begin{groupplot}[group style={
                        group name=myplot,
                        group size= 2 by 1},
                        grid style=dashed,
                        ymajorgrids=true]
       
    \nextgroupplot[stack plots=y,area style,ytick={0,5000,1000},ymin=0,ymax=10000,xmin=0.3,xmax=16.63,axis x line*=bottom, axis y line*=left,xtick={2,14},xticklabels={\Large $\epsilon'=\epsilon\phantom{\sqrt{d}}$, \Large $\epsilon'=\epsilon\times\sqrt{d}$},xtick style={draw=none}]
        \addplot table [x=eps,y=linf_ball] {graphs/data/ball_l2_base.dat}\closedcycle;
        \addplot table [x=eps,y=callote] {graphs/data/ball_l2_base.dat}\closedcycle;
        \addplot table [x=eps,y=outside] {graphs/data/ball_l2_base.dat}\closedcycle;

    \nextgroupplot[stack plots=y,area style,ytick={0,5000,1000},ymin=0,ymax=10000,xmin=0.3,xmax=16.63,axis x line*=bottom, axis y line*=left,xtick={2,14},xticklabels={\Large $\epsilon'=\epsilon\phantom{\sqrt{d}}$, \Large $\epsilon'=\epsilon\times\sqrt{d}$},xtick style={draw=none}]
        \addplot table [x=eps,y=linf_ball] {graphs/data/ball_l2_at.dat}\closedcycle;
        \addplot table [x=eps,y=callote] {graphs/data/ball_l2_at.dat}\closedcycle;
        \addplot table [x=eps,y=outside] {graphs/data/ball_l2_at.dat}\closedcycle;

    \end{groupplot}
\end{tikzpicture}
